\documentclass[letterpaper]{article} 
\usepackage[preprint]{aaai2027}  
\usepackage[hyphens]{url}  
\usepackage{graphicx} 
\usepackage{natbib}  
\usepackage{caption} 
\usepackage{algorithm}
\usepackage{algorithmic}
\usepackage{amsmath}
\usepackage{booktabs}
\usepackage{multirow}
\usepackage[table]{xcolor}
\usepackage{tcolorbox}
\tcbuselibrary{breakable}

\newcommand{\method}{ICSD}
\newcommand{\methodfull}{Influence Calibration for Self-Distillation}
\newcommand{\legacyicsd}{\method{}}
\newtheorem{proposition}{Proposition}
\newcommand{\gigpo}{GiGPO}
\newcommand{\opsd}{OPSD}
\definecolor{tablegroup}{gray}{0.93}
\definecolor{sdarrow}{RGB}{250,239,230}
\definecolor{oidrow}{RGB}{232,241,250}
\definecolor{chalframe}{RGB}{46,111,172}
\newtcolorbox{challengebox}{breakable,colback=oidrow!55!white,colframe=chalframe!45!white,boxrule=0.45pt,arc=1.5mm,left=4pt,right=4pt,top=3pt,bottom=3pt}

\title{Trust Is Not Enough: Influence Calibration for On-Policy Self-Distillation in Agentic RL}
\author{
Qizhen Lan\textsuperscript{\rm 1}\equalcontrib,
Xi Xiao\textsuperscript{\rm 2}\equalcontrib,
Xiangchen Guan\textsuperscript{\rm 3}\equalcontrib,
Mengchen Fan\textsuperscript{\rm 2},\\
Moule Lin\textsuperscript{\rm 4},
Jung Im Choi\textsuperscript{\rm 5},
Lijing Zhu\textsuperscript{\rm 6}\corresponding
}
\affiliations{
\textsuperscript{\rm 1}The University of Texas Health Science Center at Houston\\
\textsuperscript{\rm 2}University of Alabama at Birmingham\\
\textsuperscript{\rm 3}Amazon\\
\textsuperscript{\rm 4}Trinity College Dublin and Lero\\
\textsuperscript{\rm 5}The University of Findlay\\
\textsuperscript{\rm 6}University of Houston-Clear Lake\\
\texttt{Qizhen.Lan@uth.tmc.edu, Zhul@uhcl.edu}
}

\begin{document}

\maketitle

\begin{abstract}
On-policy self-distillation (OPSD) gives language agents dense token-level supervision from a privileged self-teacher on the policy's own trajectories. Existing methods allocate this supervision mainly by teacher trust, but trust does not reveal whether emphasizing a token supports the current policy objective. We call this the trust--utility mismatch and introduce \methodfull{} (\method{}). For each supervised token, \method{} measures the first-order response of its importance-weighted RL surrogate contribution to a teacher-directed output perturbation. Batch-adaptive calibration converts this non-stationary signal into a bounded allocation weight while preserving the original auxiliary-loss mass within each action turn. These detached weights affect only the distillation loss and require no additional model pass. Across ALFWorld, WebShop, and Search-QA, \method{} improves all matched aggregate metrics over trust-only allocation under Group Relative Policy Optimization (GRPO) and Group-in-Group Policy Optimization (GiGPO), across two model families spanning 1.5B to 7B. At 7B, it reaches 96.1\% ALFWorld success and a WebShop score of 93.1. Frozen-batch analyses show that \method{} reduces teacher-supported mass assigned to objective-opposed tokens from 60.1\% to 37.8\% and raises cosine compatibility with the RL gradient by 0.192.
A companion repository is available at \url{https://github.com/lanqz7766/Influence-Calibration-for-On-Policy-Self-Distillation-in-Agentic-RL}.
\end{abstract}

\begin{figure}[!t]
\centering
\includegraphics[width=0.80\columnwidth]{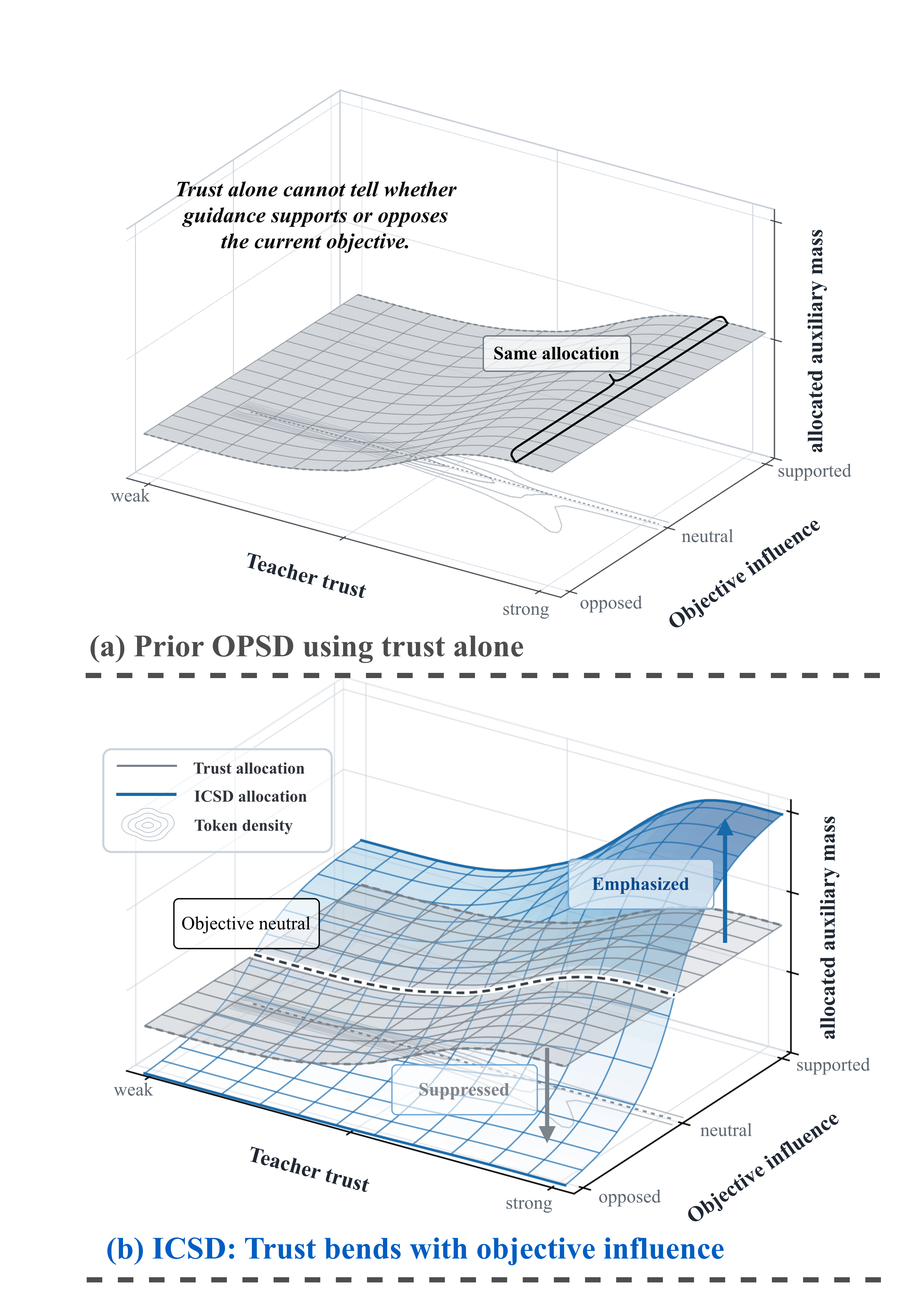}
\caption{Many existing OPSD allocation rules determine token weights using teacher trust alone. (a) Tokens with the same trust therefore receive the same coefficient, even when their influence on the RL objective differs. (b) \method{} retains the trust signal and uses objective influence to redistribute the same auxiliary mass.}
\label{fig:icsd_taste_intro}
\end{figure}

\section{Introduction}

Language agents are increasingly trained with on-policy reinforcement learning (RL) for long-horizon interaction. The reward carries the task objective but arrives once per trajectory, far more coarsely than the sequence of decisions that produced it: a delayed outcome rates the whole episode even when a few actions decide success. Recent work closes this granularity gap with dense auxiliary supervision from a privileged reference~\citep{wang2026skillsd,lu2026sdar}. The density has a price. An auxiliary distillation loss records what the reference prefers at each token, but by itself does not reveal whether moving toward that preference helps the objective currently being optimized.

On-policy self-distillation (OPSD) supplies exactly this kind of dense signal: it evaluates a teacher or privileged branch on the student's own rollouts and supervises tokens at states the current policy visits~\citep{zhao2026opsd}. Skill-conditioned variants extend this idea to multi-turn agents by exposing the teacher, but not the student, to trajectory-derived skills or other privileged context~\citep{wang2026skillsd,lu2026sdar}. This design reduces the distribution mismatch of offline teacher trajectories and turns sparse outcome feedback into token-level guidance. It also creates a different problem: once a rollout contains many teacher-supervised items, how should the auxiliary update be distributed among them?

Existing selection rules primarily estimate whether a teacher signal is informative or trustworthy from teacher--student agreement~\citep{lu2026sdar}, uncertainty and disagreement~\citep{xu2026tip}, token position~\citep{liu2026pwopsd}, or outcome evidence~\citep{akhondzadeh2026rgopd}. Useful as these quantities are, they do not determine the utility of an auxiliary update for the current RL objective. Two sampled tokens can receive equally strong teacher support while one sits in a turn the advantage estimator credits and the other in a turn it penalizes. We call this gap the \emph{trust--utility mismatch}: trust asks whether a correction is credible, utility asks whether moving toward it serves the current policy-improvement direction, and neither answers the other. On frozen ALFWorld batches, for example, trust-only allocation places 60.1\% of its teacher-supported mass on tokens the current objective opposes (Figure~\ref{fig:continuous_joint_allocation}).
We therefore condition trusted supervision on objective utility rather than replacing the teacher-trust signal.

Turning this comparison into an allocation rule raises three challenges:
\begin{challengebox}
\begin{itemize}
\setlength\itemsep{2pt}
\item \emph{Influence estimation.} Classical influence functions answer how upweighting one item changes an objective, but their parameter-space evaluation needs inverse-curvature information and per-item gradients, prohibitive during online training of billion-parameter agents~\citep{koh2017influence}.
\item \emph{Distribution calibration.} Local sensitivities inherit the drifting scale and heavy tails of advantages, ratios, and teacher gaps; a fixed threshold means different things across turns and training stages.
\item \emph{Signal validity.} A local signal is directional only where the deployed loss can realize the analyzed direction; elsewhere, sign agreement is not evidence, and the allocator must stay conservative.
\end{itemize}
\end{challengebox}

We address these challenges with \methodfull{} (\method{}), an objective-informed allocation rule for the existing OPSD auxiliary branch. \method{} obtains its influence signal from a teacher-directed perturbation of the sampled output. The score is the first-order Taylor term of the token's local RL-surrogate response and uses only quantities already available from the policy update. A batch-adaptive map converts these drifting, heavy-tailed scores into bounded relative multipliers that remain comparable across turns and training stages. Where the local interpretation is ambiguous, \method{} retains the inherited SDAR allocation~\citep{lu2026sdar}. Exact action-turn mass matching then ensures that influence redistributes trusted supervision without increasing its total amount.

Our contributions are:
\begin{itemize}
    \item We identify the trust--utility mismatch in OPSD and formulate token-level distillation as objective-informed allocation of a detached auxiliary update.
    \item We derive a teacher-directed objective-influence score and turn it into a continuous allocation rule through batch-adaptive calibration, a conservative SDAR fallback, and exact action-turn mass matching, without an additional model pass.
    \item Across three agent benchmarks, two policy optimization algorithms, and five model configurations, \method{} yields consistent aggregate gains over trust-only allocation, while frozen-batch analyses verify the predicted reallocation from objective-opposed to objective-supported teacher corrections.
\end{itemize}

\section{Related Work}

\paragraph{Privileged and Selective On-Policy Distillation.}
\opsd{} evaluates a stop-gradient privileged self-teacher on student rollouts~\citep{zhao2026opsd}. Skill-SD supplies trajectory-derived skills to the privileged branch, SDAR gates its loss by the teacher--student gap, and SAGE-OPD and TurnOPD account for where supervision occurs in a long interaction~\citep{wang2026skillsd,lu2026sdar,zhou2026sageopd,zhou2026turnopd}. Other selectors estimate trust from entropy and divergence, token position or prefix discrepancy, Bayesian evidence, or contrasted privileged hints~\citep{xu2026tip,liu2026pwopsd,xie2026iwopd,shen2026credit,pan2026rlcsd}; privileged supervision can also erase high-entropy forks and self-correction~\citep{kaur2026rethinking}. These methods determine whether and where teacher evidence is trustworthy. \method{} retains that signal but asks whether the resulting token update supports the local RL objective.

\paragraph{Coupling Distillation to the RL Objective.}
RL feedback enters prior distillation methods through several interfaces. RLSD, EGRSD, and PBSD modify RL advantages~\citep{yang2026rlsd,ke2026egrsd,tian2026pbsd}; SG-OPD uses a binary sequence-level gate, DOPD changes the supervision source and form, RG-OPD filters trajectories by reward--teacher agreement, and CRAFT adds an actor loss from sibling-rollout credit~\citep{xu2026sgopd,yu2026dopd,akhondzadeh2026rgopd,meng2026craft}. \method{} instead leaves the RL surrogate unchanged and uses its existing advantage and ratio only to allocate the detached \opsd{} loss per token. This output-coordinate view is related to gradient balancing~\citep{yu2020pcgrad,chen2018gradnorm,du2018auxiliary}, while influence functions motivate the item-upweighting question and Fisher geometry its idealized parameter-space interpretation~\citep{koh2017influence,schulman2015trpo}.

\section{Method}
\label{sec:method}

\method{} is defined directly on the privileged-distillation objective. Starting from its token-level form, we derive objective influence and use it to construct calibrated coefficients that preserve the inherited auxiliary mass. Figure~\ref{fig:icsd_method_overview} summarizes this allocation pipeline from privileged token scoring to the \method{}-weighted actor update.

\begin{figure*}[t]
\centering
\includegraphics[width=\textwidth]{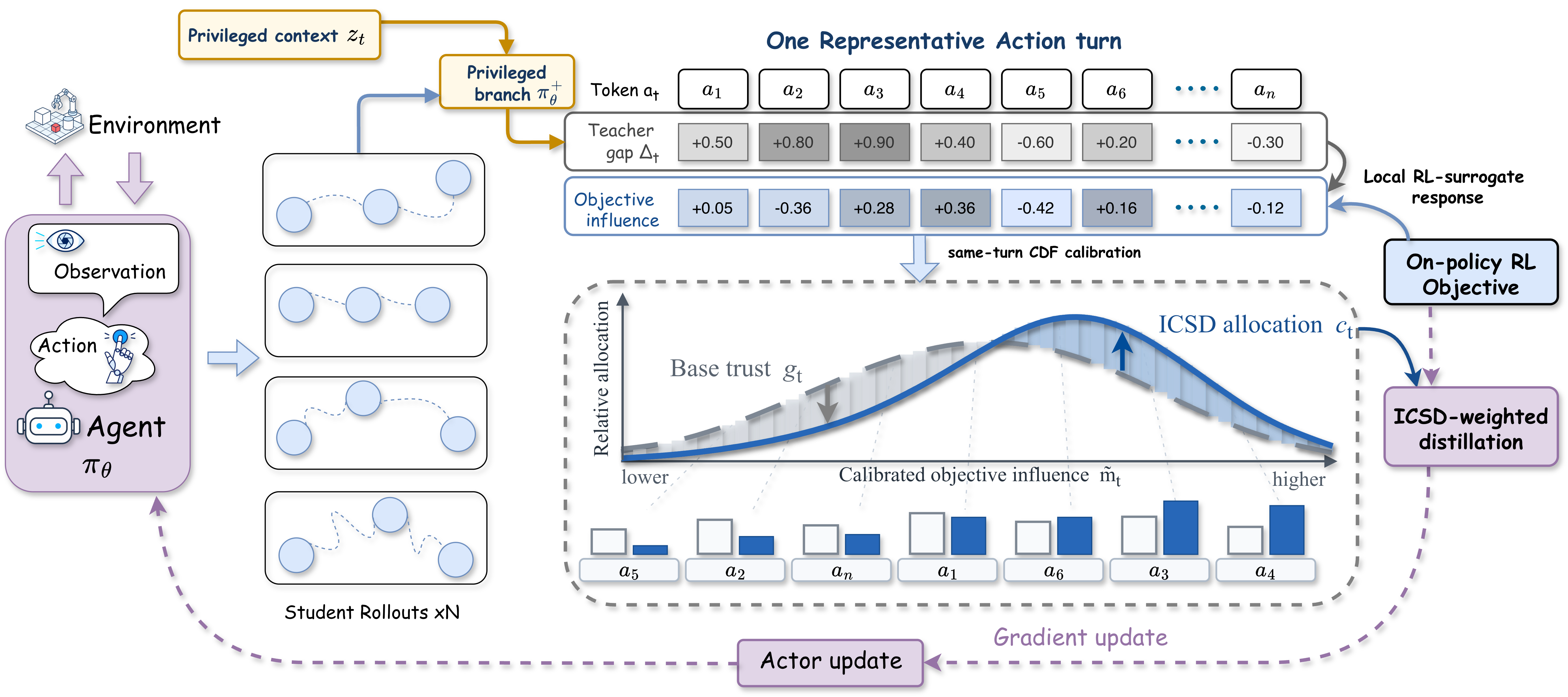}
\caption{Overview of \method{}. A privileged branch evaluates the student's on-policy action tokens, while the RL objective provides token-level objective influence. \method{} calibrates this influence and uses it to redistribute the inherited trust allocation before the actor update.}
\label{fig:icsd_method_overview}
\end{figure*}

\subsection{On-Policy Privileged Distillation}
\label{sec:method_setup}

Let $\mathcal{B}$ be the set of valid response-token positions in an on-policy minibatch. At position $t\in\mathcal{B}$, the student $\pi_\theta$ observes context $x_t$ and samples token $a_t$. The privileged branch $\pi_\theta^{+}(\cdot\mid x_t,z_t)$ evaluates that token with additional context $z_t$, such as a trajectory-derived skill, that is withheld from the student. Both branches share the evolving actor parameters. During the current actor update, however, the privileged log probability is treated as a constant. The per-token auxiliary loss is
\begin{equation}
\hat\ell^{\mathrm{KD}}_t(\theta)
=
\operatorname{sg}\!\bigl[\log \pi_\theta^{+}(a_t\mid x_t,z_t)\bigr]
-
\log \pi_\theta(a_t\mid x_t),
\label{eq:sampled_aux}
\end{equation}
where $\operatorname{sg}$ denotes stop-gradient. We also record the detached teacher--student gap
\begin{equation}
\Delta_t
=
\operatorname{sg}\!\bigl[
\log \pi_\theta^{+}(a_t\mid x_t,z_t)
-
\log \pi_\theta(a_t\mid x_t)
\bigr],
\label{eq:teacher_gap}
\end{equation}
A positive $\Delta_t$ means that the privileged branch assigns greater probability to the sampled token than the student does.

Uniform OPSD averages these token losses equally. More generally, let $c_t\geq 0$ be a detached coefficient that controls how much auxiliary weight token $t$ receives:
\begin{equation}
\mathcal{L}(\theta)
=
\mathcal{L}_{\mathrm{RL}}(\theta)
+
\frac{\lambda}{|\mathcal{B}|}
\sum_{t\in\mathcal{B}}
\operatorname{sg}(c_t)\,
\hat\ell^{\mathrm{KD}}_t(\theta).
\label{eq:allocation_form}
\end{equation}
Here $\lambda$ is the distillation weight and $|\mathcal{B}|$ is the number of valid response tokens in the minibatch. SDAR derives $c_t$ from $\Delta_t$, treating the sampled gap as evidence of whether the teacher signal should be trusted~\citep{lu2026sdar}. We denote this inherited trust weight by $g_t=G_{\mathrm{trust}}(\Delta_t)\in[0,1]$, where $G_{\mathrm{trust}}$ is monotone. In our experiments, $G_{\mathrm{trust}}$ is the same asymmetric Laplace CDF defined below, applied to $\Delta_t$ with split parameter $0.4$. ICSD leaves this SDAR component unchanged. Because $g_t$ depends only on $\Delta_t$, two tokens with the same gap receive the same trust weight. Their relation to the current RL objective is not represented.

\subsection{Teacher-Directed Objective Influence}
\label{sec:method_influence}

To add the missing objective-side information, we measure how the current RL surrogate responds when the privileged branch perturbs the sampled token. Let $\pi_{\mathrm{old}}$ be the rollout policy held fixed during the actor update. The policy ratio is $\rho_t=\pi_\theta(a_t\mid x_t)/\pi_{\mathrm{old}}(a_t\mid x_t)$, and $\widehat A_t$ is the token advantage supplied by the policy optimization algorithm. Token $t$ contributes $J_t=\widehat A_t\rho_t$ to the unclipped maximization surrogate.

Classical parameter-space influence requires applying inverse curvature to a per-item parameter gradient, typically through a linear solve or repeated Hessian--vector products~\citep{koh2017influence}. Repeating that computation for every supervised token inside each actor update is impractical for billion-parameter agents. We instead perturb the scalar output coordinate specified by the teacher--student gap:
\begin{equation}
\log\rho_t(\epsilon)
=
\log\rho_t+\epsilon\Delta_t,
\qquad
\rho_t(\epsilon)
=
\rho_t\exp(\epsilon\Delta_t).
\label{eq:intervention}
\end{equation}

The resulting teacher-directed objective influence is
\begin{align}
u_t
&:=
\left.
\frac{\mathrm d\,\widehat A_t\rho_t(\epsilon)}
     {\mathrm d\epsilon}
\right|_{\epsilon=0}
=
\widehat A_t\rho_t\Delta_t,
\label{eq:influence_score}
\\
J_t(\epsilon)
&=
J_t(0)+\epsilon u_t+\mathcal{O}(\epsilon^2).
\label{eq:influence_expansion}
\end{align}
Thus, $u_t$ is the exact first-order Taylor coefficient for the teacher-directed output-coordinate intervention in Equation~\eqref{eq:intervention}. For teacher-supported tokens ($\Delta_t>0$), its sign indicates whether increasing the sampled-token probability raises or lowers the local surrogate contribution. Its magnitude varies with the advantage and policy ratio, so it must be calibrated before being used for allocation.

\subsection{Calibrated Influence Allocation}
\label{sec:method_calibration}

\paragraph{Turn-conditioned calibration.}
The absolute scale of $u_t$ is not stable across minibatches because it inherits the scale of the advantage and policy ratio. Its product form can also yield heavy tails and unequal dispersion below and above the batch center. We therefore use the median to limit outlier leverage and fit separate scales on the two sides. Recall that $\mathcal{B}$ contains all valid response-token positions in the current minibatch. An action turn is the span emitted for one agent action in one trajectory. For each turn index, the calibration group $\mathcal{G}\subseteq\mathcal{B}$ pools valid tokens at that index across trajectories. Sparse groups use the statistics of all tokens in $\mathcal{B}$.

\paragraph{Calibration map.}
For each $\mathcal{G}$, we estimate a location and two one-sided scales:
\begin{equation}
\begin{aligned}
\hat\mu_{\mathcal{G}}
&=
\operatorname{median}_{j\in\mathcal{G}}u_j,\\
\hat b_{\mathcal{G}}^{-}
&=
\operatorname{mean}_{j\in\mathcal{G}:u_j<\hat\mu_{\mathcal{G}}}
\left|u_j-\hat\mu_{\mathcal{G}}\right|,\\
\hat b_{\mathcal{G}}^{+}
&=
\operatorname{mean}_{j\in\mathcal{G}:u_j\geq\hat\mu_{\mathcal{G}}}
\left|u_j-\hat\mu_{\mathcal{G}}\right|.
\end{aligned}
\label{eq:calibration_estimators}
\end{equation}
An empty side is assigned the scale $\varepsilon$, and every fitted scale is floored at the same small $\varepsilon>0$. The calibrated map is
\begin{equation}
\Phi_{\mathcal{G}}(u)
=
\begin{cases}
\eta\exp\!\left((u-\hat\mu_{\mathcal{G}})/\hat b_{\mathcal{G}}^{-}\right),
& u<\hat\mu_{\mathcal{G}},\\
1-(1-\eta)\exp\!\left(-(u-\hat\mu_{\mathcal{G}})/\hat b_{\mathcal{G}}^{+}\right),
& u\geq\hat\mu_{\mathcal{G}},
\end{cases}
\label{eq:cdf}
\end{equation}
Equation~\eqref{eq:cdf} is an asymmetric Laplace-family CDF with split parameter $\eta=0.5$. We set $\tilde m_t=\Phi_{\mathcal{G}}(u_t)\in(0,1)$. It is a bounded relative score within $\mathcal{G}$, not an event probability.

\paragraph{Conservative composition.}
The calibrated score is not meaningful under every sign pattern. The auxiliary loss in Equation~\eqref{eq:sampled_aux} raises the sampled log probability, while the intervention in Equation~\eqref{eq:intervention} points downward when $\Delta_t<0$. If $\widehat A_t<0$ as well, their product makes $u_t$ positive even though the two directions disagree. We retain the inherited SDAR coefficient in this case. Let $\mathcal{D}$ denote this fallback set and let $\mathcal{C}$ denote the remaining tokens:
\begin{equation}
\mathcal{D}
=
\{t\in\mathcal{B}:\widehat A_t<0,\ \Delta_t<0\},
\qquad
\mathcal{C}
=
\mathcal{B}\setminus\mathcal{D}.
\label{eq:fallback_sets}
\end{equation}
For mass matching, let $q\subseteq\mathcal{B}$ denote one action turn and let $\mathcal{Q}$ collect all turns in the minibatch. These turns form a disjoint partition, $\mathcal{B}=\biguplus_{q\in\mathcal{Q}}q$. For each $q\in\mathcal{Q}$, write $\mathcal{C}_q=\mathcal{C}\cap q$ and $\mathcal{D}_q=\mathcal{D}\cap q$. Thus, $\mathcal{G}$ pools a turn index across trajectories for calibration, whereas $q$ refers to one specific turn for mass matching. When $\sum_{j\in\mathcal{C}_q}g_j\tilde m_j>0$, we set
\begin{align}
\alpha_q
&=
\frac{\sum_{j\in\mathcal{C}_q}g_j}
{\sum_{j\in\mathcal{C}_q}g_j\tilde m_j},
\label{eq:turn_mass_scale}\\
c_t
&=
\begin{cases}
g_t, & t\in\mathcal{D}_q,\\
\alpha_q g_t\tilde m_t, & t\in\mathcal{C}_q.
\end{cases}
\label{eq:conservative_multiplier}
\end{align}
If the denominator is numerically zero, we use $c_t=g_t$ throughout that turn. The resulting allocation has the following properties.

\begin{proposition}[Invariant mass allocation]
\label{prop:equivariance}
Assume nondegenerate fitted scales and $\eta\in(0,1)$. For fixed $\mathcal{G}$, $\Phi_{\mathcal{G}}$ is monotone in $u$. Consider a positive affine transformation $u'_j=au_j+b$, $a>0$, applied to every score used to fit a calibration group, while holding the trust weights and fallback partition fixed. Then $\tilde m'_j=\tilde m_j$. When the denominator in Equation~\eqref{eq:turn_mass_scale} is nonzero, $\alpha'_q=\alpha_q$, and in either case $c'_t=c_t$. The coefficients also satisfy
\begin{equation}
\sum_{t\in q}c_t
=
\sum_{t\in q}g_t
\qquad
\text{for every }q\in\mathcal{Q}.
\label{eq:exact_turn_mass}
\end{equation}
For $t,t'\in\mathcal{C}_q$ with $g_t=g_{t'}$, $u_t\leq u_{t'}$ implies $c_t\leq c_{t'}$.
\end{proposition}

\emph{Proof sketch.}
The two branches of Equation~\eqref{eq:cdf} have positive derivatives and meet at $\eta$, so $\Phi_{\mathcal{G}}$ is monotone. A positive affine transform preserves each score's side of the median. The fitted location and one-sided scales transform as $a\hat\mu_{\mathcal{G}}+b$ and $a\hat b_{\mathcal{G}}^\pm$, leaving $\tilde m_t$ unchanged. With the trust weights and fallback partition fixed, the sums defining $\alpha_q$ are also unchanged. The same is therefore true of $\alpha_q$ and $c_t$. Expanding the two parts of each turn gives Equation~\eqref{eq:exact_turn_mass}, and equal-trust tokens inherit the ordering of $\tilde m_t$. Supplementary Section~A gives the full derivation.

\subsection{Training Objective}
\label{sec:training_objective}

The final actor update uses Equation~\eqref{eq:allocation_form} with the coefficient from Equation~\eqref{eq:conservative_multiplier}. Both $c_t$ and the privileged log probability inside $\hat\ell_t^{\mathrm{KD}}$ are detached, so auxiliary gradients pass only through the student log probability. Equation~\eqref{eq:exact_turn_mass} preserves SDAR's coefficient mass within each action turn, and computing $c_t$ requires no additional model evaluation. Supplementary Section~C reports the measured overhead.

\begin{figure*}[!t]
\centering
\includegraphics[width=\textwidth]{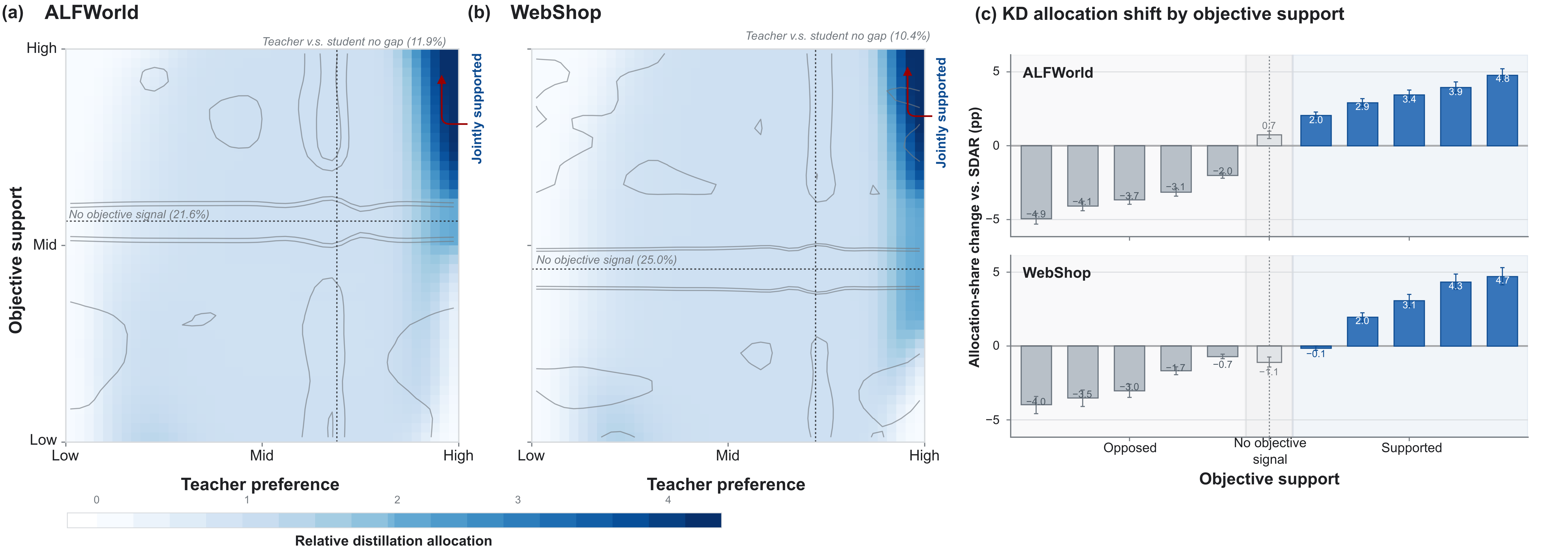}
\caption{Continuous allocation induced by \method{} on frozen GiGPO batches. Panels (a--b) show normalized coefficients over teacher-preference and objective-support ranks; darker blue denotes more mass, contours show token density, and dotted lines mark zero signals. Panel (c) reports the shift relative to SDAR among teacher-supported tokens. Whiskers are 95\% action-turn bootstrap intervals over four batches.}
\label{fig:continuous_joint_allocation}
\end{figure*}

\section{Experiments and Results}
\label{sec:experiments}

\subsection{Setup}
\label{sec:setup}

We evaluate language agents on ALFWorld~\citep{shridhar2021alfworld}, WebShop~\citep{yao2022webshop}, and the seven Search-QA subsets used by Search-R1~\citep{jin2025searchr1}. We use the task splits and metrics of SDAR~\citep{lu2026sdar}. Experiments cover Qwen2.5-Instruct models at 1.5B, 3B, and 7B~\citep{yang2024qwen25}, as well as Qwen3-1.7B and Qwen3-4B~\citep{yang2025qwen3}. We evaluate \method{} with both GRPO~\citep{shao2024deepseekmath} and GiGPO~\citep{feng2025gigpo}.

Within each matched comparison, the model, data, RL algorithm, privileged self-teacher, and inherited SDAR trust allocation are fixed. Published baseline rows retain the values reported by SDAR; the remaining entries come from our training and evaluation logs. Supplementary Table~A1 summarizes the shared SDAR-aligned training protocol.

\begin{table*}[!t]
\centering
\scriptsize
\setlength{\tabcolsep}{2.0pt}
\renewcommand{\arraystretch}{0.80}
\resizebox{\textwidth}{!}{%
\begin{tabular}{lrrrrrrrrrrrrrrrrr}
\toprule
& \multicolumn{7}{c}{ALFWorld} & \multicolumn{8}{c}{Search-QA} & \multicolumn{2}{c}{WebShop} \\
\cmidrule(lr){2-8}\cmidrule(lr){9-16}\cmidrule(lr){17-18}
Method & Pick & Look & Clean & Heat & Cool & Pick2 & Avg & NQ & Triv & Pop & Hotp & 2Wk & MuS & Bam & Avg & Score & Acc. \\
\midrule
\rowcolor{tablegroup}
\multicolumn{18}{l}{\emph{Qwen2.5-1.5B-Instruct}} \\
Skill-SD & 88.9 & 57.1 & 60.0 & 80.0 & 65.0 & 43.5 & 68.0 & 16.7 & 35.4 & 14.8 & 16.7 & 25.2 & 3.9 & 11.2 & 17.7 & 80.6 & 68.8 \\
RLSD & 78.1 & 41.7 & 68.2 & 64.3 & 77.3 & 57.7 & 67.2 & 13.9 & 29.4 & 12.4 & 14.2 & 21.1 & 3.0 & 9.5 & 14.8 & 80.5 & 63.3 \\
\cmidrule(lr){1-18}
GRPO & 84.4 & 41.7 & 77.3 & 50.0 & 68.2 & 50.0 & 65.6 & 16.6 & 34.1 & 15.9 & 16.1 & 24.8 & 3.4 & 58.1 & 24.1 & 72.0 & 60.9 \\
+ OPSD & 90.6 & 33.3 & 77.3 & 57.1 & 77.3 & 61.5 & 71.1 & 16.9 & 34.3 & 16.4 & 16.3 & 24.8 & 3.4 & 56.9 & \underline{24.2} & 84.9 & 73.4 \\
+ SDAR & 87.1 & 64.7 & 75.0 & 66.7 & 50.0 & 80.0 & 72.7 & 16.1 & 32.7 & 15.8 & 15.9 & 24.4 & 3.3 & 4.4 & 16.1 & 81.9 & 64.1 \\
\rowcolor{oidrow}
+ \legacyicsd{} & 87.5 & 71.4 & 81.5 & 62.5 & 75.8 & 52.4 & 73.4 & 18.3 & 39.2 & 17.1 & 18.7 & 28.2 & 4.4 & 68.0 & \textbf{27.7} & 85.6 & \underline{75.8} \\
\cmidrule(lr){1-18}
\gigpo{} & 97.1 & 75.0 & 95.8 & 88.9 & 84.2 & 91.7 & 91.4 & 16.8 & 35.3 & 14.6 & 17.1 & 24.5 & 3.6 & 12.8 & 17.8 & 83.9 & 73.4 \\
+ SDAR & 97.2 & 88.9 & 100.0 & 94.1 & 95.0 & 66.7 & \underline{92.2} & 16.5 & 34.4 & 14.4 & 16.7 & 24.7 & 3.1 & 11.2 & 17.3 & \underline{88.8} & 71.1 \\
\rowcolor{oidrow}
+ \legacyicsd{} & 93.1 & 91.7 & 92.0 & 100.0 & 95.5 & 92.3 & \textbf{93.8} & 17.5 & 36.8 & 14.5 & 17.7 & 24.4 & 4.1 & 12.7 & 18.2 & \textbf{92.2} & \textbf{78.1} \\
\midrule
\rowcolor{tablegroup}
\multicolumn{18}{l}{\emph{Qwen2.5-3B-Instruct}} \\
Skill-SD & 88.2 & 50.0 & 96.2 & 52.4 & 65.0 & 57.9 & 73.4 & 44.4 & 60.4 & 44.0 & 39.5 & 40.4 & 15.4 & 64.9 & 44.1 & 75.9 & 64.0 \\
RLSD & 87.9 & 75.0 & 90.9 & 75.0 & 73.1 & 68.4 & 79.7 & 41.5 & 58.6 & 42.3 & 40.4 & 40.2 & 16.8 & 66.9 & 43.8 & 84.4 & 66.4 \\
\cmidrule(lr){1-18}
GRPO & 91.2 & 62.5 & 96.2 & 61.9 & 65.0 & 47.4 & 75.0 & 39.3 & 60.6 & 41.1 & 37.4 & 34.6 & 15.4 & 26.4 & 36.4 & 79.8 & 63.3 \\
+ OPSD & 100.0 & 82.4 & 85.7 & 75.0 & 70.0 & 60.0 & 81.2 & 44.9 & 61.2 & 45.2 & 40.4 & 38.5 & 16.0 & 66.1 & \textbf{44.6} & 77.8 & 66.4 \\
+ SDAR & 97.1 & 62.5 & 100.0 & 61.9 & 75.0 & 84.2 & 84.4 & 44.8 & 58.1 & 44.3 & 38.6 & 36.2 & 15.7 & 66.1 & \underline{43.4} & 85.0 & 68.0 \\
\cmidrule(lr){1-18}
\gigpo{} & 100.0 & 75.0 & 95.5 & 78.6 & 86.4 & 100.0 & 92.2 & 42.0 & 59.5 & 42.4 & 36.9 & 37.0 & 12.6 & 64.1 & 42.1 & 86.3 & \underline{75.0} \\
+ SDAR & 100.0 & 83.3 & 96.6 & 100.0 & 79.2 & 91.3 & \underline{93.0} & 44.7 & 60.0 & 44.6 & 37.6 & 32.9 & 13.9 & 64.5 & 42.6 & \underline{87.2} & 74.2 \\
\rowcolor{oidrow}
+ \legacyicsd{} & 100.0 & 92.9 & 100.0 & 86.7 & 95.0 & 91.3 & \textbf{95.3} & 43.0 & 60.3 & 43.9 & 38.2 & 40.0 & 13.4 & 64.5 & 43.3 & \textbf{88.7} & \textbf{75.8} \\
\midrule
\rowcolor{tablegroup}
\multicolumn{18}{l}{\emph{Qwen2.5-7B-Instruct}} \\
Skill-SD & 93.9 & 93.8 & 90.9 & 100.0 & 69.2 & 68.4 & 85.1 & 47.1 & 64.5 & 47.8 & 44.2 & 42.1 & 20.2 & 69.0 & 47.8 & 86.1 & 76.5 \\
RLSD & 100.0 & 87.5 & 92.3 & 58.8 & 80.0 & 65.2 & 82.0 & 46.8 & 63.0 & 44.4 & 45.5 & 48.9 & 21.5 & 73.0 & 49.0 & 87.4 & 77.3 \\
\cmidrule(lr){1-18}
GRPO & 91.2 & 87.5 & 96.2 & 81.0 & 65.0 & 57.9 & 81.2 & 45.1 & 63.7 & 44.0 & 43.6 & 43.2 & 16.8 & 37.6 & 42.0 & 80.9 & 72.6 \\
+ OPSD & 91.4 & 61.5 & 100.0 & 87.5 & 76.5 & 52.2 & 80.4 & 47.3 & 64.5 & 46.9 & 43.8 & 39.3 & 18.0 & 69.4 & 47.0 & 86.8 & 76.5 \\
+ SDAR & 94.7 & 75.0 & 100.0 & 86.7 & 68.2 & 78.9 & 85.9 & 46.3 & 63.5 & 48.2 & 43.8 & 48.4 & 19.6 & 73.0 & \underline{49.0} & \underline{89.4} & \underline{82.8} \\
\cmidrule(lr){1-18}
\gigpo{} & 100.0 & 58.3 & 90.9 & 100.0 & 81.8 & 100.0 & 91.4 & 46.4 & 64.7 & 46.1 & 41.6 & 43.6 & 18.9 & 68.9 & 47.2 & 84.0 & 71.1 \\
+ SDAR & 100.0 & 87.0 & 100.0 & 87.5 & 90.9 & 95.2 & \underline{94.5} & 45.1 & 63.8 & 46.3 & 41.8 & 41.8 & 19.4 & 72.2 & 47.2 & 88.4 & 78.9 \\
\rowcolor{oidrow}
+ \legacyicsd{} & 100.0 & 91.7 & 90.9 & 100.0 & 95.5 & 96.2 & \textbf{96.1} & 47.8 & 64.6 & 48.5 & 44.2 & 45.3 & 20.1 & 73.3 & \textbf{49.1} & \textbf{93.1} & \textbf{84.4} \\
\midrule
\rowcolor{tablegroup}
\multicolumn{18}{l}{\emph{Qwen3-1.7B-Instruct}} \\
Skill-SD & 52.9 & 37.5 & 69.2 & 42.9 & 60.0 & 36.8 & 52.3 & 39.1 & 57.5 & 45.4 & 34.8 & 34.1 & 10.7 & 64.1 & 40.8 & \textbf{81.8} & 53.9 \\
RLSD & 50.0 & 37.5 & 61.5 & 19.0 & 50.0 & 21.1 & 42.2 & 38.6 & 57.3 & 43.0 & 34.5 & 34.1 & 11.5 & 65.3 & 40.6 & 74.0 & 50.8 \\
\cmidrule(lr){1-18}
GRPO & 71.1 & 41.7 & 36.4 & 40.0 & 31.8 & 31.6 & 46.1 & 40.0 & 58.9 & 43.5 & 35.4 & 30.3 & 12.0 & 65.7 & 40.8 & 67.3 & 38.3 \\
+ OPSD & 38.2 & 50.0 & 30.8 & 28.6 & 30.0 & 21.1 & 32.0 & 40.7 & 58.9 & 45.0 & 37.0 & 34.6 & 13.3 & 65.7 & \underline{42.2} & 70.7 & 38.3 \\
+ SDAR & 73.5 & 25.0 & 76.9 & 33.3 & 40.0 & 36.8 & 53.9 & 39.7 & 58.9 & 45.3 & 35.9 & 35.5 & 12.6 & 65.3 & 41.9 & 76.8 & 58.6 \\
\cmidrule(lr){1-18}
\gigpo{} & 85.3 & 50.0 & 96.2 & 66.7 & 75.0 & 68.4 & \underline{78.1} & 43.1 & 58.4 & 44.8 & 33.8 & 28.9 & 10.9 & 62.9 & 40.4 & 79.9 & 60.2 \\
+ SDAR & 72.7 & 62.5 & 90.9 & 58.3 & 84.6 & 68.4 & 75.0 & 42.1 & 59.1 & 46.7 & 34.8 & 29.7 & 11.0 & 63.7 & 41.0 & 78.4 & \underline{65.6} \\
\rowcolor{oidrow}
+ \legacyicsd{} & 94.1 & 62.5 & 100.0 & 71.4 & 85.0 & 57.9 & \textbf{82.8} & 43.6 & 59.7 & 47.0 & 36.3 & 31.3 & 13.2 & 66.1 & \textbf{42.5} & \underline{81.5} & \textbf{68.0} \\
\bottomrule
\end{tabular}%
}
\caption{Main results grouped by policy optimizer. Skill-SD and RLSD are hybrid distillation baselines; family-head rows are the RL baselines, ``+'' rows add auxiliary allocators, and blue rows denote \method{}. All values are percentages. Within each model block, bold and underline mark the best and second-best summary results; task-level columns are unmarked.}
\label{tab:main_sdar_aligned}
\end{table*}

\begin{table}[tb]
\centering
\footnotesize
\setlength{\tabcolsep}{4.0pt}
\begin{tabular}{lcccc}
\toprule
 & ALFWorld & Search-QA & \multicolumn{2}{c}{WebShop} \\
\cmidrule(lr){2-2}\cmidrule(lr){3-3}\cmidrule(lr){4-5}
Method & Avg & Avg & Score & Acc. \\
\midrule
\rowcolor{tablegroup}
\multicolumn{5}{l}{\emph{Qwen2.5-1.5B-Instruct}} \\
Vanilla & 5.5 & 12.1 & 17.8 & 5.5 \\
Skill-Prompt* & 6.2 & 7.7 & 20.8 & 1.6 \\
OPSD & 14.1 & 0.0 & 22.3 & 10.2 \\
\rowcolor{oidrow}
\gigpo{}+\legacyicsd{} & \textbf{93.8} & \textbf{18.2} & \textbf{92.2} & \textbf{78.1} \\
\midrule
\rowcolor{tablegroup}
\multicolumn{5}{l}{\emph{Qwen2.5-3B-Instruct}} \\
Vanilla & 21.9 & 31.7 & 6.7 & 0.8 \\
Skill-Prompt* & 28.9 & 23.9 & 0.2 & 0.8 \\
OPSD & 28.1 & 0.0 & 11.3 & 3.1 \\
Skill-GRPO & 60.2 & 34.1 & 77.3 & 60.9 \\
Skill-GRPO* & 80.5 & 36.1 & 76.3 & 66.4 \\
\rowcolor{oidrow}
\gigpo{}+\legacyicsd{} & \textbf{95.3} & \textbf{43.3} & \textbf{88.7} & \textbf{75.8} \\
\midrule
\rowcolor{tablegroup}
\multicolumn{5}{l}{\emph{Qwen2.5-7B-Instruct}} \\
Vanilla & 12.5 & 33.9 & 5.9 & 1.6 \\
Skill-Prompt* & 23.4 & 36.4 & 1.7 & 0.8 \\
OPSD & 32.8 & 6.2 & 4.5 & 2.3 \\
Skill-GRPO & 69.5 & 40.3 & 80.4 & 71.9 \\
Skill-GRPO* & 88.3 & 47.5 & 87.0 & 81.2 \\
\rowcolor{oidrow}
\gigpo{}+\legacyicsd{} & \textbf{96.1} & \textbf{49.1} & \textbf{93.1} & \textbf{84.4} \\
\midrule
\rowcolor{tablegroup}
\multicolumn{5}{l}{\emph{Qwen3-1.7B-Instruct}} \\
Vanilla & 12.5 & 24.8 & 46.5 & 4.7 \\
Skill-Prompt* & 9.4 & 24.3 & 23.0 & 2.3 \\
OPSD & 14.1 & 5.8 & 47.4 & 9.3 \\
Skill-GRPO & 21.1 & 40.4 & 73.4 & 46.1 \\
Skill-GRPO* & 28.1 & 40.7 & 80.4 & 50.0 \\
\rowcolor{oidrow}
\gigpo{}+\legacyicsd{} & \textbf{82.8} & \textbf{42.5} & \textbf{81.5} & \textbf{68.0} \\
\bottomrule
\end{tabular}
\caption{Summary results for extended baselines. Starred methods retain retrieved skills during validation. Supplementary Table~A2 reports per-task and per-subset results.}
\label{tab:extended_baselines}
\end{table}

\subsection{Main Results}

Table~\ref{tab:main_sdar_aligned} provides the matched comparison. Replacing SDAR's trust-only allocation with \method{} improves every summary metric in the GiGPO blocks, and the 1.5B GRPO control shows the same ordering. The Qwen2.5-7B result illustrates the scale of the difference: ALFWorld success rises from 94.5 to 96.1, while WebShop score/accuracy moves from 88.4/78.9 to 93.1/84.4.

Table~\ref{tab:extended_baselines} broadens the comparison to Vanilla and Skill-Prompt, stand-alone OPSD~\citep{zhao2026opsd}, and the Skill-GRPO variants used by SDAR~\citep{lu2026sdar}. \gigpo{}+\method{} gives the strongest ALFWorld and Search-QA averages and the highest WebShop accuracy. Skill-Prompt* and Skill-GRPO* retain retrieved skills during evaluation, whereas \method{} requires no privileged context at test time.

\subsection{Ablations and Analysis}

\paragraph{Across policy optimizers.}
To test whether the benefit depends on the policy optimizer, we hold the model and task settings fixed at Qwen2.5-1.5B and compare the full training trajectories under GRPO and GiGPO (Supplementary Figure~A3). \method{} outperforms the matched SDAR allocator on ALFWorld and WebShop with either optimizer. The gap is largest on WebShop accuracy: $+11.7$ points with GRPO and $+7.0$ with GiGPO. The result is therefore not specific to GiGPO's step-aware advantage construction.

\paragraph{Qwen3 model family.}
Table~\ref{tab:qwen3_transfer} repeats the matched comparison on Qwen3-4B-Instruct~\citep{yang2025qwen3}. With GiGPO and the concise non-thinking action interface fixed, \method{} improves average ALFWorld success from 89.8 to 95.3 over SDAR. This 5.5-point gain complements the Qwen2.5 results without changing the allocation rule.

\begin{table}[tb]
\centering
\scriptsize
\setlength{\tabcolsep}{3.4pt}
\begin{tabular}{lrrrrrrr}
\toprule
Method & Pick & Look & Clean & Heat & Cool & Pick2 & Avg \\
\midrule
\gigpo{}+SDAR & 96.9 & 66.7 & \textbf{100.0} & 78.6 & 81.8 & \textbf{96.2} & 89.8 \\
\rowcolor{oidrow}
\gigpo{}+\legacyicsd{} & \textbf{97.6} & \textbf{91.7} & \textbf{100.0} & \textbf{88.9} & \textbf{100.0} & 83.3 & \textbf{95.3} \\
\bottomrule
\end{tabular}
\caption{Qwen3-4B ALFWorld results with GiGPO and the concise non-thinking interface.}
\label{tab:qwen3_transfer}
\end{table}
\begin{table}[tb]
\centering
\scriptsize
\setlength{\tabcolsep}{3.0pt}
\begin{tabular}{lrrrrrrr}
\toprule
Method & Pick & Look & Clean & Heat & Cool & Pick2 & Avg \\
\midrule
\gigpo{} & 97.1 & 75.0 & 95.8 & 88.9 & 84.2 & 91.7 & 91.4 \\
+ SDAR & 97.2 & 88.9 & \textbf{100.0} & 94.1 & 95.0 & 66.7 & 92.2 \\
+ Fisher magnitude & \textbf{100.0} & 78.6 & 95.0 & \textbf{100.0} & 90.0 & 78.3 & 91.4 \\
+ Sign-only & 97.1 & \textbf{100.0} & 90.9 & \textbf{100.0} & \textbf{95.5} & 77.3 & 92.2 \\
+ Influence only & 93.9 & 88.9 & 92.3 & \textbf{100.0} & \textbf{95.5} & 83.3 & 93.0 \\
\rowcolor{oidrow}
+ \legacyicsd{} & 93.1 & \textbf{91.7} & 92.0 & \textbf{100.0} & \textbf{95.5} & \textbf{92.3} & \textbf{93.8} \\
\bottomrule
\end{tabular}
\caption{Allocation ablation on ALFWorld with Qwen2.5-1.5B and GiGPO.}
\label{tab:ablation}
\end{table}

\begin{figure}[tb]
\centering
\includegraphics[width=\columnwidth]{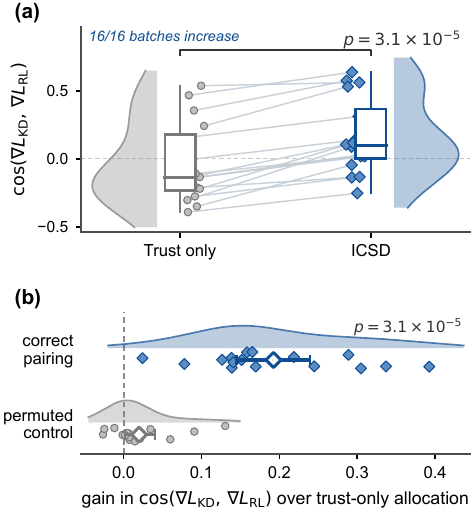}
\caption{\method{} improves local update compatibility on sixteen frozen batches. (a) SDAR and \method{} cosine compatibility with the RL gradient. (b) Compatibility after tokenwise permutation of the \method{} multipliers. Diamonds and horizontal bars show means and 95\% batch-bootstrap intervals.}
\label{fig:gradient_alignment_flow}
\end{figure}

\paragraph{Allocation ablation.}
Table~\ref{tab:ablation} separates three parts of the allocation rule: signed objective relevance, teacher trust, and continuous influence magnitude. The Fisher-magnitude control passes an unsigned policy-gradient magnitude through the same CDF calibration; the teacher gap appears only in the inherited SDAR trust weight. Its 91.4\% average matches GiGPO, showing that sensitivity magnitude alone is not enough. Influence-only allocation instead retains the signed score but removes token-level trust, raising the average to 93.0\%. Combining both signals gives the full 93.8\% result.

The sign-only control replaces the continuous calibrated score with a hard positive-or-negative decision while retaining the fallback and exact turn-wise mass matching. It reaches 92.2\%, tying SDAR but remaining 1.6 points below \method{}. Thus, the gain is not reproduced by unsigned sensitivity or hard sign filtering alone; the strongest result comes from combining teacher trust with graded, signed objective relevance. Supplementary Figures~A1--A2 visualize fallback behavior and selected validation trajectories.

\paragraph{Sensitivity.}
Figure~\ref{fig:icsd_sensitivity} motivates the settings used in our main experiments. The central CDF split and $\lambda=0.01$ give the strongest late-stage performance in the tested ranges, and asymmetric Laplace calibration finishes above the Gaussian alternative. Continuous signed influence also reaches a higher peak and late-window mean than magnitude-only or sign-only allocation.

\begin{figure}[tb]
\centering
\includegraphics[width=\columnwidth]{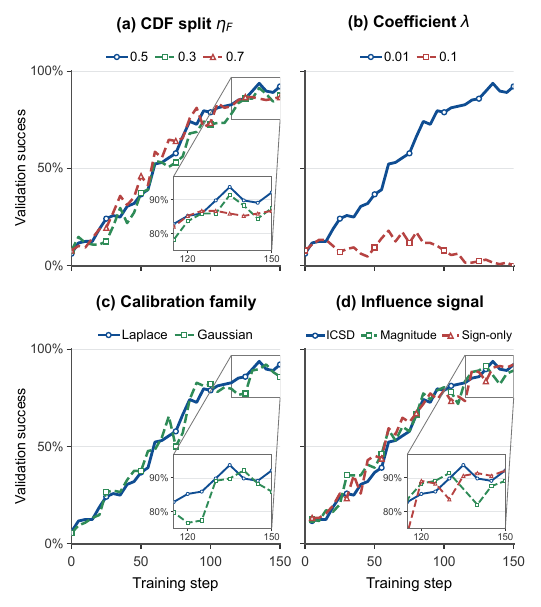}
\caption{Training dynamics for sensitivity and signal ablations on ALFWorld with Qwen2.5-1.5B and GiGPO. Insets enlarge steps 115--150.}
\label{fig:icsd_sensitivity}
\end{figure}

\paragraph{Mechanism analysis.}
For nonnegative token weights $w_t$, trusted-conflict mass (TCM) is the fraction of active teacher-supported mass ($\Delta_t>0$, $|u_t|>\varepsilon$) assigned to objective-opposed tokens ($u_t<0$). The exact estimator and four-region breakdown are provided in Supplementary Section~D and Figure~A1. Across four matched ALFWorld batches, TCM falls from $60.1\%$ under SDAR to $37.8\%$ under \method{}, a 22.3-point reduction (95\% CI $[22.1,22.6]$). Figure~\ref{fig:continuous_joint_allocation} shows the same redistribution continuously: strongly trusted and supported regions receive $2.90\times$ and $2.64\times$ mass on ALFWorld and WebShop, while trusted but opposed ALFWorld tokens receive $0.79\times$.

Across sixteen frozen batches, \method{} raises cosine compatibility with the RL gradient by $0.192$ (95\% CI $[0.147,0.240]$), whereas permuting the same multipliers leaves only a $0.020$ gain (Figure~\ref{fig:gradient_alignment_flow}). The compatibility gain is positive in all sixteen batches; Supplementary Section~D reports the paired bootstrap analysis. The improvement therefore comes from token assignment rather than total mass.

\section{Conclusion}

\method{} addresses the trust--utility mismatch in on-policy self-distillation by reallocating trusted token supervision according to teacher-directed objective influence. Batch-adaptive calibration, conservative fallback, and exact action-turn mass matching keep the RL objective and auxiliary budget unchanged. Across three benchmarks, two optimizers, and five model configurations, \method{} consistently improves over trust-only allocation; frozen-batch analyses show that it shifts mass toward objective-supported corrections and improves local update compatibility.

\bibliography{references}

\appendix
\setcounter{secnumdepth}{2}
\setcounter{figure}{0}
\renewcommand{\thefigure}{A\arabic{figure}}
\setcounter{table}{0}
\renewcommand{\thetable}{A\arabic{table}}
\setcounter{algorithm}{0}
\renewcommand{\thealgorithm}{A\arabic{algorithm}}
\setcounter{equation}{0}
\renewcommand{\theequation}{A\arabic{equation}}

\section{Derivations and Proofs}
\label{app:proofs}

This section gives the explicit Taylor remainder for objective influence and proves the affine-invariance statement from the main paper.

\paragraph{Why not parameter-space influence.}
For a training item $t$, classical influence has the form $-\nabla_\theta J^\top H^{-1}\nabla_\theta\ell_t$, where $H$ is the parameter Hessian~\citep{koh2017influence}. Computing it for every supervised token requires per-token parameter gradients and an inverse-curvature solve, usually approximated with repeated Hessian--vector products. This cost would be incurred inside every actor update. The objective influence score used by \method{} instead differentiates a scalar output coordinate using quantities already produced by the RL update.

\subsection{Taylor Remainder for Objective Influence}

Fix a token $t$ and abbreviate $A=\widehat A_t$, $\rho=\rho_t>0$, and $\Delta=\Delta_t$; all three are constants of the current iterate, and $\Delta$ is detached by construction. The perturbed surrogate contribution is $J(\epsilon)=A\rho\,e^{\epsilon\Delta}$, which is smooth in $\epsilon$. Differentiating at zero gives
\begin{equation}
J'(0)=A\rho\Delta=u_t ,
\end{equation}
which is exactly the objective influence score $u_t$ defined in the main paper. Applying Taylor's theorem with the Lagrange remainder to $e^{x}$ at $x=\epsilon\Delta$, there exists $\xi$ between $0$ and $\epsilon\Delta$ with $e^{\epsilon\Delta}=1+\epsilon\Delta+\tfrac{1}{2}(\epsilon\Delta)^2e^{\xi}$. Substituting,
\begin{equation}
J(\epsilon)-J(0)-\epsilon u_t
=
\frac{A\rho\Delta^2}{2}\,\epsilon^2 e^{\xi},
\end{equation}
and hence, bounding $e^{\xi}$ by $e^{|\epsilon\Delta|}$,
\begin{equation}
\bigl|J(\epsilon)-J(0)-\epsilon u_t\bigr|
\le
\frac{|A|\rho\Delta^2}{2}\,\epsilon^2 e^{|\epsilon\Delta|}.
\label{eq:remainder_bound}
\end{equation}
The $\mathcal{O}(\epsilon^2)$ term in the main paper's first-order expansion therefore carries the explicit constant $\tfrac{1}{2}|A|\rho\Delta^2e^{|\epsilon\Delta|}$, and no approximation enters the first-order coefficient itself.

\subsection{Monotonicity and Affine Invariance}

For nondegenerate $\hat b_-$ and $\hat b_+$, differentiating the two branches of the asymmetric Laplace CDF defined in the main paper gives
\begin{equation}
\frac{\mathrm d\Phi_{\mathcal{G}}}{\mathrm du}
=
\begin{cases}
\Phi_{\mathcal{G}}(u)/\hat b_-, & u<\hat\mu,\\
\bigl(1-\Phi_{\mathcal{G}}(u)\bigr)/\hat b_+, & u\geq\hat\mu.
\end{cases}
\end{equation}
Both derivatives are positive, and the branches meet at $\Phi_{\mathcal{G}}(\hat\mu)=\eta$. The map is therefore continuous and strictly increasing.

Write $\mathcal{J}_{-}=\{j:u_j<\hat\mu\}$ and $\mathcal{J}_{+}=\{j:u_j\geq\hat\mu\}$. Let $u'_j=au_j+b$ with $a>0$. The median is affine equivariant, so $\hat\mu'=a\hat\mu+b$. The side assignment is preserved because $u'_j<\hat\mu'$ if and only if $u_j<\hat\mu$. Hence, $\mathcal{J}'_{\pm}=\mathcal{J}_{\pm}$. Each one-sided scale transforms as
\begin{equation}
\hat b'_{\pm}
=\operatorname{mean}_{j\in\mathcal{J}_{\pm}}\bigl|u'_j-\hat\mu'\bigr|
=a\hat b_{\pm},
\end{equation}
because $|u'_j-\hat\mu'|=a|u_j-\hat\mu|$. The standardized residual is therefore invariant:
\begin{equation}
\frac{u'-\hat\mu'}{\hat b'_{\sigma(u')}}
=
\frac{a(u-\hat\mu)}{a\,\hat b_{\sigma(u)}}
=
\frac{u-\hat\mu}{\hat b_{\sigma(u)}},
\end{equation}
where $\sigma(u)$ selects the side of the median. Substitution into the asymmetric Laplace CDF gives $\Phi'_{\mathcal{G}}(u')=\Phi_{\mathcal{G}}(u)$ and hence $\tilde m'_t=\tilde m_t$. With $g_t$, $\mathcal{C}_q$, and $\mathcal{D}_q$ fixed, both sums in the turnwise scale are unchanged. It follows that $\alpha'_q=\alpha_q$ whenever the denominator is nonzero and $c'_t=c_t$. If the denominator is zero, both constructions use $c_t=g_t$. The scale floor is only a numerical safeguard. A binding floor is the degenerate case excluded by the affine-invariance statement.

\subsection{Turn-Level Conservation and Ordering}

Fix an action turn $q$ with $\sum_{j\in\mathcal{C}_q}g_j\tilde m_j>0$. By the turnwise scale and conservative multiplier defined in the main paper,
\begin{equation}
\begin{aligned}
\sum_{t\in q}c_t
&=
\sum_{t\in\mathcal{D}_q}g_t
+
\alpha_q\sum_{t\in\mathcal{C}_q}g_t\tilde m_t\\
&=
\sum_{t\in\mathcal{D}_q}g_t
+
\sum_{t\in\mathcal{C}_q}g_t
=
\sum_{t\in q}g_t ,
\end{aligned}
\end{equation}
which is the exact turn-mass identity stated in the main paper. In the degenerate case, the rule sets $c_t=g_t$ throughout the turn, so the identity holds directly.

Summing over all turns gives the global coefficient-mass identity
\begin{equation}
\sum_{t\in\mathcal{B}}c_t=\sum_{t\in\mathcal{B}}g_t
\qquad\text{for every minibatch,}
\label{eq:global_budget}
\end{equation}
so \method{} does not increase the total coefficient budget relative to the inherited trust allocation.

For two tokens $t,t'\in\mathcal{C}_q$ in the same turn,
\begin{equation}
\frac{c_t}{c_{t'}}
=
\frac{g_t\tilde m_t}{g_{t'}\tilde m_{t'}},
\end{equation}
because the common factor $\alpha_q$ cancels. At equal trust, the within-turn allocation therefore orders tokens by $\tilde m_t=\Phi_{\mathcal{G}}(u_t)$, which is monotone in $u_t$. In particular, opposing signs of $\widehat A_t$ and $\Delta_t$ give $u_t<0$ and a lower calibrated score than positive-influence items. The shared positive factor $\alpha_q$ preserves this ordering, though it does not require every negative-influence coefficient to be smaller than its base trust coefficient. This is the ordering used in the main paper's mechanism-analysis figure.

\section{Idealized Natural-Gradient Interpretation}
\label{app:interpretation}

This section relates the exact output-coordinate score $u_t$ to an idealized parameter-space update. The argument motivates its sign; it is not part of the deployed algorithm.

\paragraph{Setting.}
Let $s_t=\nabla_\theta\log\pi_\theta(a_t\mid x_t)$ be the sampled-token score and let $F$ denote the Fisher information of the policy on visited states, the local metric underlying trust-region policy optimization~\citep{schulman2015trpo}. Consider the idealized teacher-directed step
\begin{equation}
\delta\theta
=
\eta F^{-1}\Delta_t s_t,
\qquad
\eta>0.
\label{eq:app_step}
\end{equation}
Assume that the unclipped surrogate is locally active at $t$, $F$ is positive definite on the span of $s_t$, and cross-token score correlations under $F^{-1}$ are neglected. A first-order expansion of the token's own contribution $J_t=\widehat A_t\rho_t$ gives
\begin{equation}
J_t(\theta+\delta\theta)-J_t(\theta)
=
\eta\kappa_tu_t+o(\eta),
\qquad
\kappa_t=s_t^\top F^{-1}s_t\geq0.
\label{eq:app_expansion}
\end{equation}
Thus, whenever $\kappa_t>0$, the same-token local contribution changes with the sign of $u_t$. The leverage $\kappa_t$ is token dependent, so this argument establishes local sign compatibility rather than preservation of the full score ranking. It neither models cross-token interactions nor guarantees finite-horizon return improvement.

\paragraph{Relation to the fallback.}
The sign interpretation inherits the support restriction of the teacher-directed intervention defined in the main paper. When $\Delta_t<0$, the analyzed intervention and the sampled-token auxiliary update point in different directions. If $\widehat A_t<0$ as well, their product makes $u_t$ positive and can falsely resemble supportive evidence. The conservative multiplier avoids this ambiguity by retaining the base trust coefficient in that double-negative region. Opposing-sign cases still yield $u_t<0$ and remain useful as relative conflict scores.

\section{Implementation Details}
\label{app:implementation}

\subsection{Training Protocol}

\paragraph{Protocol details.}
We use the SDAR data splits, SkillBank, keyword-matching retrieval, and privileged self-teacher. GRPO and GiGPO retain their original advantage construction; the GiGPO runs use $\gamma=0.95$ and unit step-advantage weight. The ICSD-specific trust and influence split parameters are 0.4 and 0.5, with a scale floor of $10^{-4}$ and a minimum calibration group of eight tokens. Sparse groups use minibatch statistics before reverting to the fixed monotone map. Model-dependent micro-batches and tensor parallelism are chosen to fit the corresponding model sizes.

\begin{table}[t]
\centering
\scriptsize
\setlength{\tabcolsep}{2.2pt}
\renewcommand{\arraystretch}{1.03}
\begin{tabular}{@{}lccc@{}}
\toprule
Setting & ALFWorld & WebShop & Search-QA \\
\midrule
Train / validation batch & 16 / 128 & 16 / 128 & 128 / 512 \\
Rollouts per task & 8 & 8 & 8 \\
Maximum interaction steps & 50 & 15 & 4 \\
Maximum prompt / response length & 2048 / 512 & 4096 / 512 & 4096 / 512 \\
Actor mini-batch size & 256 & 64 & 256 \\
Actor learning rate & $1\times10^{-6}$ & $1\times10^{-6}$ & $1\times10^{-6}$ \\
Ratio clip & 0.2 & 0.2 & 0.2 \\
KL coefficient & 0.01 & 0.01 & 0.001 \\
Distillation coefficient $\lambda$ & 0.01 & 0.01 & 0.01 \\
ICSD training updates & 150 & 150 & 200 \\
\bottomrule
\end{tabular}
\caption{SDAR-aligned protocol used for ICSD training.}
\label{tab:training_configuration}
\end{table}

\subsection{Runtime and Approximation Audit}

\paragraph{Compute overhead.}
\method{} adds no forward or backward pass; its cost is limited to detached elementwise arithmetic and CDF fitting inside the actor update. On matched Qwen2.5-1.5B WebShop runs using four H100 GPUs per method, \method{} increases per-token actor-update time by 1.7\% over the trust-only SDAR allocator. Since actor updates account for roughly one fifth of step time, this corresponds to about 0.4\% normalized end-to-end overhead.

\paragraph{Clipping exposure.}
The influence score is defined on the unclipped surrogate term, while the deployed loss clips the policy ratio. Across the four frozen ALFWorld evidence batches (544{,}427 valid response tokens), 0.128\% of tokens fall outside the clip interval and 0.069\% lie where the clipped surrogate is locally flat, so that $u_t$ overstates the local response. These tokens receive 0.103\% of the composite coefficient mass and 0.984\% of the absolute weighted distillation mass $c_t|\Delta_t|$; no token reaches the negative-advantage dual-clip cap. The flat-token fraction is stable across batches (0.058--0.075\%), and the pooled ratio distribution concentrates near one (mean 1.0006, standard deviation 0.0254). These statistics describe the analyzed step-135 batches rather than every training stage; within them, the unclipped approximation misprices only a negligible portion of the auxiliary update.

\subsection{Allocator Implementation}

\paragraph{Trust and influence calibration.}
The inherited trust role $g_t=G_{\mathrm{trust}}(\Delta_t)$ is instantiated as a fitted asymmetric location--scale CDF of the sampled gap. Raw influence values inherit the heavy tails and drifting scale of the advantage, policy ratio, and teacher--student gap. Their supported and opposed sides can also have different spreads. The ICSD-specific map $\Phi_{\mathcal{G}}$ therefore uses the same distribution family with a robust location estimate and separate one-sided scales. Each scale is floored away from zero. Fits use valid response tokens from the current calibration group. Undersized or degenerate groups fall back first to global minibatch statistics and then to a fixed monotone map. The statistics are detached and recomputed for every minibatch. We fit the CDF before applying the fallback rule. The remaining coefficients are then rescaled within each action turn so that their composite mass, together with the unchanged fallback coefficients, exactly equals the turn's base trust mass. Algorithm~\ref{alg:icsd} summarizes the update.

\begin{algorithm}[tbp]
\caption{\method{} update for one on-policy minibatch}
\label{alg:icsd}
\begin{algorithmic}[1]
\REQUIRE minibatch $\mathcal{B}$ of valid response tokens grouped into action turns; distillation weight $\lambda$; scale floor $\varepsilon$
\STATE Roll out $\pi_\theta$; compute $\widehat A_t$, $\rho_t$, and the sampled student log probabilities.
\STATE Evaluate the stop-gradient privileged branch $\pi_\theta^{+}$ on the same contexts with input $z_t$; compute $\Delta_t$ and the base trust coefficient $g_t=G_{\mathrm{trust}}(\Delta_t)$.
\STATE Compute the detached influence $u_t=\widehat A_t\rho_t\Delta_t$.
\STATE Fit $(\hat\mu,\hat b_{-},\hat b_{+})$ on the calibration group by the median and one-sided means floored at $\varepsilon$; set $\tilde m_t=\Phi_{\mathcal{G}}(u_t)$. Undersized groups fall back to minibatch statistics.
\STATE Form $\mathcal{D}=\{t:\widehat A_t<0,\ \Delta_t<0\}$ and $\mathcal{C}=\mathcal{B}\setminus\mathcal{D}$.
\FOR{each action turn $q$}
\IF{$\sum_{j\in\mathcal{C}_q}g_j\tilde m_j>\varepsilon$}
\STATE $\alpha_q\leftarrow\bigl(\sum_{j\in\mathcal{C}_q}g_j\bigr)\big/\bigl(\sum_{j\in\mathcal{C}_q}g_j\tilde m_j\bigr)$
\STATE $c_t\leftarrow g_t$ on $\mathcal{D}_q$; \ $c_t\leftarrow\alpha_q\,g_t\tilde m_t$ on $\mathcal{C}_q$
\ELSE
\STATE $c_t\leftarrow g_t$ for all $t\in q$ \COMMENT{degenerate turn: retain base trust}
\ENDIF
\ENDFOR
\STATE Detach $\{c_t\}$ and update $\theta$ with the allocation objective.
\ENSURE $\sum_{t\in q}c_t=\sum_{t\in q}g_t$ for every action turn
\end{algorithmic}
\end{algorithm}

\section{Full Signal-Combination Audit}
\label{app:signal_audit}

Figure~\ref{fig:quadrant_allocation_appendix} compares base trust with the deployed \method{} allocation in the four nonzero sign regions. The row labels report the signs of the sampled teacher--student gap $\Delta_t$ and the policy advantage $\widehat A_t$.

\paragraph{Trusted-conflict mass.}
For nonnegative token weights $w_t$, define
\begin{equation}
\operatorname{TCM}_{w}(\mathcal{B})
=
\frac{
\sum_{t\in\mathcal{B}}
w_t\,\mathbf{1}[\Delta_t>0,\;|u_t|>\varepsilon,\;u_t<0]
}{
\sum_{t\in\mathcal{B}}
w_t\,\mathbf{1}[\Delta_t>0,\;|u_t|>\varepsilon]
}.
\label{eq:tcm}
\end{equation}
Here $\varepsilon$ removes zero-advantage tokens, $w_t=g_t$ for SDAR, and $w_t=c_t$ for \method{}. This normalized share measures allocation rather than raw coefficient scale. Across four matched ALFWorld batches, TCM falls from $60.1\%$ under SDAR to $37.8\%$ under \method{}. The paired reduction is 22.3 points (95\% action-turn bootstrap CI $[22.1,22.6]$) and appears in every batch.

\paragraph{Gradient compatibility.}
Across sixteen frozen batches, \method{} increases cosine compatibility with the RL gradient by $0.192$ on average (95\% batch-bootstrap CI $[0.147,0.240]$), with a positive change in every batch. Permuting the same multipliers across tokens preserves their distribution and coefficient mass but leaves only a $0.020$ gain. The pairing-specific difference is $0.172$ (95\% CI $[0.131,0.217]$), tying the geometric improvement to token assignment rather than aggregate scale.

\begin{figure}[t]
\centering
\includegraphics[width=\columnwidth]{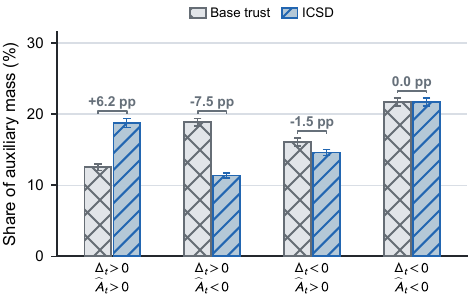}
\caption{Deployed allocation across the four sign regions. \method{} moves auxiliary mass toward jointly supported tokens and away from sign-conflicting regions, while leaving the double-negative region unchanged. Bars and whiskers show mass shares and 95\% action-turn bootstrap intervals over four ALFWorld batches (544{,}427 tokens); brackets give point differences from base trust.}
\label{fig:quadrant_allocation_appendix}
\end{figure}

Across the frozen batches, \method{} assigns 6.2 percentage points more mass to $\Delta_t>0,\widehat A_t>0$, 7.5 points less to $\Delta_t>0,\widehat A_t<0$, and 1.5 points less to $\Delta_t<0,\widehat A_t>0$. When both signals are negative, the deployed coefficient remains equal to base trust, so the difference is zero rather than a spurious positive shift.

\section{Extended Learning Curves}

Figure~\ref{fig:allocation_ablation} reports the component variants used in the main paper's ablation table.

\begin{figure*}[t]
\centering
\includegraphics[width=0.62\textwidth]{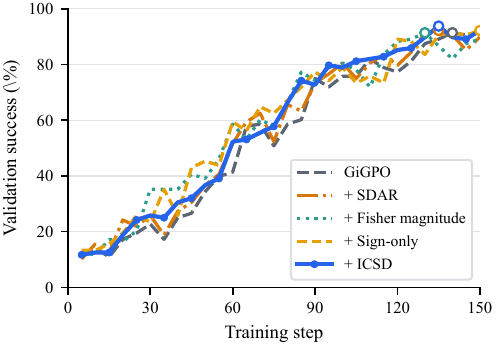}
\caption{Validation curves for matched ALFWorld allocation variants with Qwen2.5-1.5B, GiGPO, and 150 training updates under an identical training configuration.}
\label{fig:allocation_ablation}
\end{figure*}

Figure~\ref{fig:onepointfive_dynamics} adds a temporal view to the frozen-batch mechanism audit. The allocation signal is noisy at the minibatch level, especially under GRPO, yet the five-step trends remain below the corresponding SDAR traces for most of training. The difference widens over the final 30 steps under GRPO (46.6\% versus 56.1\%) and narrows under GiGPO (47.8\% versus 49.2\%).

\begin{figure*}[t]
\centering
\includegraphics[width=0.92\textwidth]{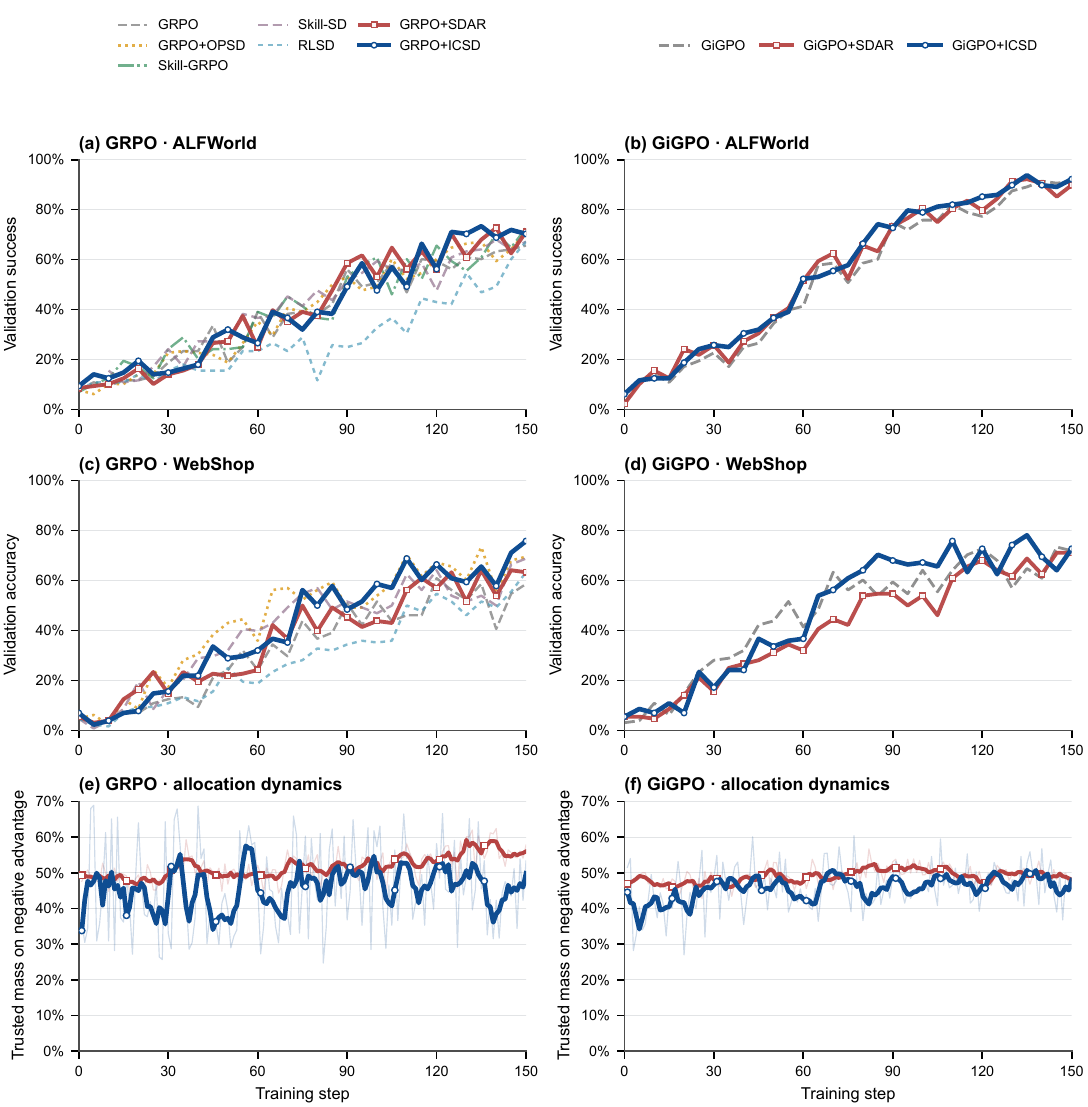}
\caption{Training dynamics at Qwen2.5-1.5B. Panels (a--b) show ALFWorld, panels (c--d) show WebShop, and panels (e--f) show the share of trusted auxiliary mass assigned to negative-advantage tokens. Pale lines are per-step values and dark lines are centered five-step means. This share describes allocation, not token utility.}
\label{fig:onepointfive_dynamics}
\end{figure*}

\section{Continuous-Landscape Construction Details}
\label{app:continuous_joint_allocation}

For panels (a--b) of the main paper's mechanism-analysis figure, let $c_t$ denote the deployed coefficient and let $q_T$ and $q_R$ be the empirical mid-ranks of $\Delta_t$ and $\widehat A_t$ within each environment. For a joint rank bin $\mathcal{B}_{ij}$, the plotted color is the locally smoothed normalized allocation intensity
\begin{equation}
\widehat I_{ij}
=
\frac{
|\mathcal{B}_{ij}|^{-1}
\sum_{t\in\mathcal{B}_{ij}} c_t
}{
|\mathcal{B}|^{-1}
\sum_{t\in\mathcal{B}} c_t
}.
\label{eq:normalized_allocation_intensity}
\end{equation}
Thus, $\widehat I_{ij}=1$ denotes the environment-average deployed allocation. The gray contours show token density and do not enter the allocator; the rank-transformed axes expose dependence between the two signals rather than their raw scale.

Panel (c) directly measures redistribution relative to the base trust allocation inside the teacher-supported set $\mathcal{S}^{+}=\{t:\Delta_t>0\}$. Negative and positive advantages are ranked separately and divided into five equal-frequency bins on each side of zero; exact zeros remain a separate neutral point and are never jittered into an artificial order. For a bin $\mathcal{K}_k$, the vertical axis is
\begin{equation}
\Delta s_k
=
100\left(
\frac{\sum_{t\in\mathcal{K}_k}c_t}
     {\sum_{t\in\mathcal{S}^{+}}c_t}
-
\frac{\sum_{t\in\mathcal{K}_k}g_t}
     {\sum_{t\in\mathcal{S}^{+}}g_t}
\right),
\label{eq:conditional_share_shift}
\end{equation}
in percentage points per bin. Confidence intervals use action-turn bootstrap resampling over the four evidence batches. Since $\mathcal{S}^{+}$ excludes $\mathcal{D}$, Equation~\eqref{eq:conditional_share_shift} describes the deployed allocation exactly within this conditional normalization.

\begin{table*}[t]
\centering
\scriptsize
\setlength{\tabcolsep}{2.0pt}
\renewcommand{\arraystretch}{0.96}
\resizebox{\textwidth}{!}{%
\begin{tabular}{llrrrrrrrrrrrrrrrrr}
\toprule
& & \multicolumn{7}{c}{ALFWorld} & \multicolumn{8}{c}{Search-QA} & \multicolumn{2}{c}{WebShop} \\
\cmidrule(lr){3-9}\cmidrule(lr){10-17}\cmidrule(lr){18-19}
Family & Method & Pick & Look & Clean & Heat & Cool & Pick2 & Avg & NQ & Triv & Pop & Hotp & 2Wk & MuS & Bam & Avg & Score & Acc. \\
\midrule
\rowcolor{tablegroup}
\multicolumn{19}{l}{\emph{Qwen2.5-1.5B-Instruct}} \\
 & Vanilla & 11.1 & 0.0 & 6.2 & 0.0 & 0.0 & 4.2 & 5.5 & 10.5 & 25.4 & 17.8 & 8.3 & 14.8 & 2.2 & 6.0 & 12.1 & 17.8 & 5.5 \\
 & Skill-Prompt* & 3.4 & 16.7 & 12.9 & 5.3 & 0.0 & 5.0 & 6.2 & 6.2 & 18.1 & 8.7 & 5.4 & 10.1 & 1.5 & 4.0 & 7.7 & 20.8 & 1.6 \\
\cellcolor{white}\multirow{-3}{*}{\textbf{No RL}} & OPSD & 26.3 & 16.7 & 9.1 & 6.7 & 9.1 & 5.3 & 14.1 & 0.0 & 0.0 & 0.0 & 0.1 & 0.2 & 0.0 & 0.0 & 0.0 & 22.3 & 10.2 \\
\cmidrule(lr){2-19}
 & Skill-SD & \underline{88.9} & \underline{57.1} & 60.0 & \underline{80.0} & 65.0 & 43.5 & \underline{68.0} & 16.7 & 35.4 & 14.8 & 16.7 & \underline{25.2} & 3.9 & 11.2 & 17.7 & \underline{80.6} & \underline{68.8} \\
\multirow{-2}{*}{\textbf{Hybrid}} & RLSD & 78.1 & 41.7 & \underline{68.2} & 64.3 & \underline{77.3} & \underline{57.7} & 67.2 & 13.9 & 29.4 & 12.4 & 14.2 & 21.1 & 3.0 & 9.5 & 14.8 & 80.5 & 63.3 \\
\rowcolor{oidrow}
\textbf{Ours} & \gigpo{}+\legacyicsd{} & \textbf{93.1} & \textbf{91.7} & \textbf{92.0} & \textbf{100.0} & \textbf{95.5} & \textbf{92.3} & \textbf{93.8} & \textbf{17.5} & \textbf{36.8} & \textbf{14.5} & \textbf{17.7} & \textbf{24.4} & \textbf{4.1} & \textbf{12.7} & \textbf{18.2} & \textbf{92.2} & \textbf{78.1} \\
\midrule
\rowcolor{tablegroup}
\multicolumn{19}{l}{\emph{Qwen2.5-3B-Instruct}} \\
 & Vanilla & 44.4 & 11.1 & 6.2 & 15.4 & 28.6 & 12.5 & 21.9 & 24.6 & 48.1 & 31.0 & 26.3 & 25.3 & 7.2 & 59.7 & 31.7 & 6.7 & 0.8 \\
 & Skill-Prompt* & 51.7 & 66.7 & 48.4 & 0.0 & 4.3 & 10.0 & 28.9 & 23.7 & 46.2 & 30.6 & 24.4 & 22.1 & 7.5 & 12.5 & 23.9 & 0.2 & 0.8 \\
\cellcolor{white}\multirow{-3}{*}{\textbf{No RL}} & OPSD & 48.8 & 41.7 & 16.7 & 0.0 & 15.8 & 16.7 & 28.1 & 0.1 & 0.1 & 0.1 & 0.0 & 0.0 & 0.0 & 0.0 & 0.0 & 11.3 & 3.1 \\
\cmidrule(lr){2-19}
 & Skill-GRPO & 88.9 & 71.4 & 58.8 & 70.6 & 40.7 & 29.2 & 60.2 & 43.5 & 58.8 & 43.0 & 36.8 & 32.2 & 11.7 & 12.5 & 34.1 & 77.3 & 60.9 \\
\multirow{-2}{*}{\textbf{Skill/RL}} & Skill-GRPO* & \underline{94.3} & 57.1 & \underline{100.0} & 66.7 & 73.1 & 57.1 & \underline{80.5} & \underline{44.3} & 59.6 & \textbf{44.3} & 39.0 & 36.1 & 14.5 & 14.9 & 36.1 & 76.3 & 66.4 \\
 & Skill-SD & 88.2 & 50.0 & 96.2 & 52.4 & 65.0 & 57.9 & 73.4 & \textbf{44.4} & \textbf{60.4} & \underline{44.0} & \underline{39.5} & \textbf{40.4} & \underline{15.4} & \underline{64.9} & \textbf{44.1} & 75.9 & 64.0 \\
\multirow{-2}{*}{\textbf{Hybrid}} & RLSD & 87.9 & \underline{75.0} & 90.9 & \underline{75.0} & \underline{73.1} & \underline{68.4} & 79.7 & 41.5 & 58.6 & 42.3 & \textbf{40.4} & \underline{40.2} & \textbf{16.8} & \textbf{66.9} & \underline{43.8} & \underline{84.4} & \underline{66.4} \\
\rowcolor{oidrow}
\textbf{Ours} & \gigpo{}+\legacyicsd{} & \textbf{100.0} & \textbf{92.9} & \textbf{100.0} & \textbf{86.7} & \textbf{95.0} & \textbf{91.3} & \textbf{95.3} & 43.0 & \underline{60.3} & 43.9 & 38.2 & 40.0 & 13.4 & 64.5 & 43.3 & \textbf{88.7} & \textbf{75.8} \\
\midrule
\rowcolor{tablegroup}
\multicolumn{19}{l}{\emph{Qwen2.5-7B-Instruct}} \\
 & Vanilla & 36.1 & 22.2 & 3.1 & 0.0 & 0.0 & 0.0 & 12.5 & 25.2 & 50.8 & 29.5 & 29.0 & 29.0 & 10.4 & 63.7 & 33.9 & 5.9 & 1.6 \\
 & Skill-Prompt* & 51.7 & 50.0 & 32.3 & 5.3 & 4.3 & 0.0 & 23.4 & 30.9 & 52.1 & 32.7 & 32.7 & 27.9 & 12.7 & 66.1 & 36.4 & 1.7 & 0.8 \\
\cellcolor{white}\multirow{-3}{*}{\textbf{No RL}} & OPSD & 50.0 & 60.0 & 22.7 & 21.4 & 17.6 & 9.5 & 32.8 & 8.8 & 8.6 & 17.5 & 2.5 & 4.2 & 0.5 & 1.2 & 6.2 & 4.5 & 2.3 \\
\cmidrule(lr){2-19}
 & Skill-GRPO & 88.5 & 66.7 & 65.2 & 61.1 & 57.7 & 73.1 & 69.5 & 45.2 & 63.7 & 45.7 & 43.1 & 43.3 & 19.6 & 21.4 & 40.3 & 80.4 & 71.9 \\
\multirow{-2}{*}{\textbf{Skill/RL}} & Skill-GRPO* & 100.0 & 83.3 & \textbf{96.4} & 83.3 & 75.0 & \underline{78.9} & \underline{88.3} & 44.8 & 63.0 & 45.1 & 43.7 & 43.7 & \underline{20.5} & 71.4 & 47.5 & 87.0 & \underline{81.2} \\
 & Skill-SD & 93.9 & \textbf{93.8} & 90.9 & \underline{100.0} & 69.2 & 68.4 & 85.1 & \underline{47.1} & \underline{64.5} & \underline{47.8} & 44.2 & 42.1 & 20.2 & 69.0 & 47.8 & 86.1 & 76.5 \\
\multirow{-2}{*}{\textbf{Hybrid}} & RLSD & \underline{100.0} & 87.5 & \underline{92.3} & 58.8 & \underline{80.0} & 65.2 & 82.0 & 46.8 & 63.0 & 44.4 & \textbf{45.5} & \textbf{48.9} & \textbf{21.5} & \underline{73.0} & \underline{49.0} & \underline{87.4} & 77.3 \\
\rowcolor{oidrow}
\textbf{Ours} & \gigpo{}+\legacyicsd{} & \textbf{100.0} & \underline{91.7} & 90.9 & \textbf{100.0} & \textbf{95.5} & \textbf{96.2} & \textbf{96.1} & \textbf{47.8} & \textbf{64.6} & \textbf{48.5} & \underline{44.2} & \underline{45.3} & 20.1 & \textbf{73.3} & \textbf{49.1} & \textbf{93.1} & \textbf{84.4} \\
\midrule
\rowcolor{tablegroup}
\multicolumn{19}{l}{\emph{Qwen3-1.7B-Instruct}} \\
 & Vanilla & 25.0 & 22.2 & 3.1 & 0.0 & 21.4 & 4.2 & 12.5 & 29.4 & 46.9 & 37.0 & 23.5 & 19.6 & 6.4 & 10.5 & 24.8 & 46.5 & 4.7 \\
 & Skill-Prompt* & 10.3 & 50.0 & 16.1 & 0.0 & 0.0 & 5.0 & 9.4 & 29.4 & 46.5 & 36.2 & 22.9 & 20.8 & 4.3 & 10.1 & 24.3 & 23.0 & 2.3 \\
\cellcolor{white}\multirow{-3}{*}{\textbf{No RL}} & OPSD & 26.3 & 33.3 & 9.1 & 0.0 & 4.5 & 5.3 & 14.1 & 4.2 & 8.3 & 4.6 & 6.6 & 15.3 & 0.7 & 1.2 & 5.8 & 47.4 & 9.3 \\
\cmidrule(lr){2-19}
 & Skill-GRPO & 27.6 & \underline{54.5} & 22.7 & 27.3 & 0.0 & 19.2 & 21.1 & \underline{39.2} & \underline{58.6} & 43.9 & 35.2 & 28.2 & 11.5 & \underline{66.1} & 40.4 & 73.4 & 46.1 \\
\multirow{-2}{*}{\textbf{Skill/RL}} & Skill-GRPO* & 31.4 & 42.9 & 51.9 & 8.3 & 41.5 & 7.1 & 28.1 & 38.0 & 58.4 & 43.9 & \textbf{36.3} & 29.0 & \underline{12.5} & \textbf{66.9} & 40.7 & 80.4 & 50.0 \\
 & Skill-SD & \underline{52.9} & 37.5 & \underline{69.2} & \underline{42.9} & \underline{60.0} & \underline{36.8} & \underline{52.3} & 39.1 & 57.5 & \underline{45.4} & 34.8 & \textbf{34.1} & 10.7 & 64.1 & \underline{40.8} & \textbf{81.8} & \underline{53.9} \\
\multirow{-2}{*}{\textbf{Hybrid}} & RLSD & 50.0 & 37.5 & 61.5 & 19.0 & 50.0 & 21.1 & 42.2 & 38.6 & 57.3 & 43.0 & 34.5 & \textbf{34.1} & 11.5 & 65.3 & 40.6 & 74.0 & 50.8 \\
\rowcolor{oidrow}
\textbf{Ours} & \gigpo{}+\legacyicsd{} & \textbf{94.1} & \textbf{62.5} & \textbf{100.0} & \textbf{71.4} & \textbf{85.0} & \textbf{57.9} & \textbf{82.8} & \textbf{43.6} & \textbf{59.7} & \textbf{47.0} & \textbf{36.3} & \underline{31.3} & \textbf{13.2} & \underline{66.1} & \textbf{42.5} & \underline{81.5} & \textbf{68.0} \\
\bottomrule
\end{tabular}%
}
\caption{Full per-task ALFWorld, per-subset Search-QA, and WebShop results corresponding to the aggregate table in the main paper. Starred methods retain retrieved skills during validation.}
\label{tab:extended_baselines_full}
\end{table*}

\end{document}